\documentclass{article}
\usepackage[utf8]{inputenc}
\usepackage[margin=1in]{geometry}
\usepackage{bm, amsmath, mathrsfs, amssymb, amsthm, amscd, amsfonts, graphicx, color,epsfig,latexsym,layout,fancyhdr, colonequals, enumerate, hyperref, fontenc, hyperref, graphics, verbatim, psfrag, lipsum}
\allowdisplaybreaks
\newcommand\numberthis{\addtocounter{equation}{1}\tag{\theequation}}
\newcommand\blfootnote[1]{
  \begingroup
  \renewcommand\thefootnote{}\footnote{#1}
  \addtocounter{footnote}{0}
  \endgroup
}
\newtheorem{theorem}{Theorem}[section]

\newtheorem{lemma}[theorem]{Lemma}

\numberwithin{equation}{section}

\title{Diffusion Maps for Group-Invariant Manifolds }
\author{Paulina Hoyos  and Joe Kileel\blfootnote{Authors' emails: \href{mailto:paulinah@utexas.edu}{paulinah@utexas.edu}, \href{mailto:jkileel@math.utexas.edu}{jkileel@math.utexas.edu}}}
\date{\today\footnote{Just prior to posting the first version of this preprint, the authors became aware of similar and simultaneous results achieved by Rosen, Cheng and Shkolnisky \cite{rosen2023}.  Both teams conducted their research independently.}}

\begin{document}

\maketitle

\begin{abstract}
\noindent In this article, we consider the manifold learning problem when the data set is invariant under the action of a compact Lie group $K$. 
Our approach consists in augmenting the data-induced graph Laplacian by integrating over the $K$-orbits of the existing data points, which yields a $K$-invariant graph Laplacian $L$.
We prove that $L$ can be diagonalized by using the unitary irreducible representation matrices of $K$, and we provide an explicit formula for computing its eigenvalues and eigenfunctions. 
In addition, we show that the normalized Laplacian operator $L_N$ converges to the Laplace-Beltrami operator of the data manifold with an improved convergence rate, where the improvement grows with the dimension of the symmetry group $K$. 
This work extends the steerable graph Laplacian framework of Landa and Shkolnisky from the case of $\operatorname{SO}(2)$ to arbitrary compact Lie groups.
\end{abstract}

\section{Introduction}
In manifold learning, a given data set $\mathcal{X} = \{x_1, \dots, x_N\} \subseteq \mathbb{R}^\mathcal{D}$ is assumed to lie on or near a low-dimensional manifold $\mathcal{M}$ of dimension $d \ll \mathcal{D}$. 
Tools in this field have important applications in data analysis, including dimensionality reduction \cite{6789755,doi:10.1126/science.290.5500.2319, zelesko2020earthmover}, semi-supervised learning \cite{pmlr-v5-goldberg09a, pmlr-v54-moscovich17a}, function approximation \cite{doi:10.1080/01621459.2013.827984, SOBER2021113140} and denoising \cite{SteerablePaper}. 
A prevalent approach to manifold learning is based on constructing a weighted graph whose vertices are the elements of $\mathcal{X}$ and where edges are assigned weights according to the affinity of the pairs of points in question. More precisely, edge $\{x_i,x_j\}$ is assigned the weight
$W_{ij} = K_{\varepsilon}(\vert\vert x_i - x_j \vert\vert),$
where $K_{\varepsilon}$ is a kernel function, commonly chosen to be the Gaussian kernel 
$W_{ij} = \exp\left(- \vert\vert x_i - x_j \vert\vert /\varepsilon\right).$ The graph Laplacian $L$ is then defined to act on functions $f \colon \mathcal{X} \to \mathbb{R}$ by
\[ Lf(x_i) = D_{ii}f(x_i) - \sum_{j=1}^N W_{ij}f(x_j),\]
where $D_{ii} = \sum_{j=1}^N W_{ij}$.
After suitable normalization, as $N \to \infty$ the graph Laplacian $L$ approximates the Laplace-Beltrami operator $\Delta_{\mathcal{M}}$ \cite{doi:10.1126/science.290.5500.2319} of the data manifold $\mathcal{M}$ and its eigenvectors converge to the eigenfunctions of $\Delta_\mathcal{M}$ \cite{NIPS2006_5848ad95}. Methods that use the graph Laplacian include the spectral embedding techniques of diffusion maps \cite{COIFMAN20065} and Laplacian eigenmaps \cite{NIPS2006_5848ad95}. 

There are important applications where the data points are each subject to an unknown rotation or reflection, or more generally the action of an unknown group element \cite{kondor2008group}. For example, consider 2D tomographic images in cryo-electron microscopy \cite{singer2020computational}, which are subject to 2D rigid transformations, or 3D subtomograms in cryo-electron tomography \cite{liu2022high}, which are subject to spatial rotations. For data sets like these, it is most natural to assume that the underlying data manifold is closed under the group action.  
The question arises: how should we incorporate symmetries into the graph Laplacian framework?

This paper gives a principled answer to this question.  
In practice data augmentation is common \cite{chen2020group}, however adding in random points from each observation's orbits runs the risk of becoming too computationally expensive. 
Instead, we extend a framework of Landa and Shkolnisky \cite{SteerablePaper} to arbitrary compact Lie groups.  This achieves infinite data augmentation efficiently, through analytic integration over the group orbits. We construct the $K$-invariant graph Laplacian $L$ where $K$ is the underlying symmetry group in Section \ref{sec: graph Laplacian}.
The main result is Theorem \ref{thm: spectrum of L}, where we characterize the eigenvalues and eigenfunctions of $L$ and show how to compute them. 
Then in Section \ref{sec: convergence of graph Laplacian} we prove Theorem \ref{thm:rate}, which states that the normalized graph Laplacian $L_N = D^{-1}L$ converges pointwise to the Laplace-Beltrami operator on $\mathcal{M}$. Moreover, we prove that the variance term of the convergence scales with $d - \mathrm{dim}(K)$ rather than $d$ in the case of symmetry-unaware diffusion maps. 
This implies an improved sample complexity rate for computations with the graph Laplacian.

\section{Setting}
Let $\mathcal{M}$ be a compact $d$-dimensional smooth manifold without boundary embedded in $\mathbb{R}^\mathcal{D}$.  Consider data points $\{ x_1, x_2, \dots, x_N \} \subseteq \mathbb{R}^\mathcal{D}$ sampled independently from the uniform probability distribution $p(x)$ on $\mathcal{M}$. 

To encode the symmetries of $\mathcal{M}$, we let $K$ be a compact Lie group acting linearly and by isometries on $\mathbb{R}^\mathcal{D}$, and we assume $\mathcal{M}$ is $K$-invariant: for every $x \in \mathcal{M}$ and $\kappa \in K$, it holds that $\kappa \cdot x \in \mathcal{M}$. Moreover, we assume that the action of $K$ on $\mathcal{M}$ is \textit{generically free}, meaning that the set of points in $\mathcal{M}$ whose stabilizer is nontrivial has measure zero. In other words, if
$\mathcal{M}^\prime$ is the subset of $\mathcal{M}$ where $K$ acts freely, then $\mathcal{M} \setminus \mathcal{M}^\prime$ has measure zero. In particular, integrating over $\mathcal{M}^\prime$ is the same as integrating over $\mathcal{M}$, and we will use this fact without further mention in Section \ref{sec: convergence of graph Laplacian}.

\section{The $K$-invariant graph Laplacian and its spectrum}\label{sec: graph Laplacian}
Take $\Gamma \colonequals \{1, 2, \dots, N\} \times K$, and consider the space $\mathcal{H} = \mathcal{L}^2(\Gamma)$ of square integrable functions $f \colon \Gamma \to \mathbb{R}$ equipped with the inner product 
\[ \langle g, f \rangle \colonequals \sum_{i=1}^N \int_{K} g^\ast(i, \kappa) f(i, \kappa) d\kappa, \]
where $d\kappa$ is the normalized Haar measure on $K$. This is an infinite-dimensional complex Hilbert space. Note that, for any $f \in \mathcal{H}$, we can write $f(i, \kappa) = f_i(\kappa)$ with $\{f_i\}_{i=1}^N \in \mathcal{L}^2(K)$. Hence, we can view a function in $\mathcal{H}$ as an $N$-dimensional vector where each entry is a square-integrable function on $K$. 

For any pair of data points $x_i, x_j$ and any pair of elements $\kappa, \lambda$ in $K$, we define the affinity between the elements $\kappa \cdot x_i$ and $\lambda \cdot x_j$ in $\mathbb{R}^\mathcal{D}$ as
\begin{equation}\label{eq: affinity of two points}
    W_{ij}(\kappa, \lambda) \colonequals \exp \left( - \vert\vert {\kappa}\cdot x_i - {\lambda} \cdot x_j \vert\vert^2 / \varepsilon \right).
\end{equation}
Then we define the affinity operator $W \colon \mathcal{H} \to \mathcal{H}$ by
\[ Wf(i, \kappa) = \sum_{j=1}^N \int_{K} W_{ij}(\kappa, \lambda) f(j, \lambda) d\lambda \]
for any function $f \in \mathcal{H}$, and any $(i, \kappa) \in \Gamma$. This operator defines a weighted graph $G = (V, E, W)$ whose vertices $V$ are the data points $\{ x_1, \dots, x_n \}$ together with all the points $\kappa \cdot x_i$, $1 \leq i \leq N$ and $\kappa \in K$, in their orbits under the action of $K$; and the set of edges $E$ consists of all the pairs $\{\kappa \cdot x_i, \lambda \cdot x_j\}$ for which $W_{i,j}(\kappa, \lambda) > 0$.
\begin{lemma}
For any group elements $\kappa, \lambda \in K$, it holds that
\begin{align*}
    W_{ij}(\kappa, \lambda) = W_{ij}(\mathrm{Id}, \kappa^{-1} \lambda),
\end{align*}
where $\mathrm{Id}$ is the identity element of $K$.
\end{lemma}
\begin{proof}
This follows from the fact that $K$ acts by isometries on $\mathbb{R}^\mathcal{D},$ since acting by $k^{-1}$ gives
\begin{align*}
    \vert\vert \kappa \cdot x_i - \lambda \cdot x_j\vert\vert^2 = \vert\vert x_i - \kappa^{-1} \lambda\cdot x_j\vert\vert^2.
\end{align*}
\end{proof}

Let $D$ be the diagonal matrix with diagonal entries given by 
\[ D_{ii} = \sum_{j=1}^N \int_{K} W_{ij}(\mathrm{Id}, \lambda) d\lambda.\]
Note that the entries $D_{ii}$ are constant. In addition, as $d\lambda$ is the Haar measure  on $K$, it is invariant under left translations, so we have that
\[ D_{ii} = \sum_{j=1}^N \int_{K} W_{ij}(\mathrm{Id}, \lambda) d\lambda = \sum_{j=1}^N \int_{K} W_{ij}(\kappa, \lambda) d\lambda. \]
That is, $D_{ii}$ is independent of the group element $\kappa$ in the first argument of $W_{ij}$.

Define the graph Laplacian $L \colon \mathcal{H} \to \mathcal{H}$ by $Lf = Df - Wf$. More explicitly, it acts on a function $f \in \mathcal{H}$ as
\[ Lf(i, \kappa) = f(i, \kappa) \sum_{j=1}^N \int_K W_{ij}(\kappa, \lambda) d\lambda  - \sum_{j=1}^N \int_K W_{ij}(\kappa, \lambda) f(j, \lambda) d\lambda. \]
\begin{lemma}\label{Claim: quadratic form of L}
    The graph Laplacian $L$ admits the quadratic form 
    \[ \langle f, Lf \rangle = \frac{1}{2 }\sum_{i=1}^N \sum_{j=1}^N \int_{K} \int_{K} W_{ij}(\kappa,\lambda) \vert f_i(\kappa) - f_j(\lambda)\vert^2 d\kappa d\lambda. \]
\end{lemma}
\begin{proof}
\begin{align*}
    \langle f, Lf \rangle 
    &= \sum_{i=1}^N \int_K f^\ast(i, \kappa) Lf(i, \kappa) d\kappa \\
    &= \sum_{i=1}^N \int_K f^\ast(i, \kappa) \left[ f(i, \kappa) \sum_{j=1}^N \int_K W_{ij}(\kappa, \lambda) d\lambda  - \sum_{j=1}^N \int_K W_{ij}(\kappa, \lambda) f(j, \lambda) d\lambda \right] \\
    &= \sum_{i=1}^N \sum_{j=1}^N \int_K \int_K W_{ij}(\kappa, \lambda) \left[ f^\ast(i, \kappa)f(i, \kappa) - f^\ast(i, \kappa)f(j,\lambda)\right] d\kappa d\lambda\\
    &= \frac{1}{2}\sum_{i=1}^N \sum_{j=1}^N \int_K \int_K W_{ij}(\kappa, \lambda) \left[ f^\ast(i, \kappa)f(i, \kappa) - 2f^\ast(i, \kappa)f(j,\lambda) + f^\ast(j,\lambda)f(j,\lambda)\right] d\kappa d\lambda\\
    &= \frac{1}{2}\sum_{i=1}^N \sum_{j=1}^N \int_K \int_K W_{ij}(\kappa, \lambda) \vert f(i,\kappa) - f(j,\lambda) \vert^2 d\kappa d\lambda.
\end{align*}    
\end{proof}
From here we conclude that $L$ is positive semidefinite, since $\langle f, Lf \rangle \geq 0$ for all $f \in \mathcal{H}$. 

We are interested in finding the eigenvalues and eigenvectors of the graph Laplacian $L$. In doing so, we need to take the Fourier transform of each  $W_{ij}(\kappa, \lambda)$, which involves the representation theory of $K$ \cite{sternberg1995group, chirikjian2016harmonic}. For each $i,j \in \{1, \dots, N\}$, consider the Fourier transform
\begin{align*}
    \widehat{W}_{ij}^{(\ell)} = \int_{K} W_{ij}(\text{Id}, \kappa^{-1} \lambda) {U_{\ell}(\kappa^{-1}\lambda)} d(\kappa^{-1} \lambda) = \int_{K} W_{ij}(\text{Id}, \kappa) {U_{\ell}(\kappa)} d\kappa,
\end{align*}
where $U_{\ell}(\kappa^{-1}\lambda)$ is the $\ell$-th irreducible unitary representation matrix for $K$. Note that $\widehat{W}_{ij}^{(\ell)}$ is an $\mathrm{dim}E_\ell \times \mathrm{dim}E_\ell$ matrix which does not depend on $K$. More explicitly, the entries $\left[ \widehat{W}_{ij}^{(\ell)} \right]_{mn}$ of $  \widehat{W}_{ij}^{(\ell)}$ are given by
\[\left[\widehat{W}_{ij}^{(\ell)} \right]_{mn} = \int_{K} W_{ij}(\text{Id}, \kappa) {U_{\ell}(\kappa)}_{mn}  d\kappa,\]
where ${U_{\ell}(\kappa)}_{mn}$ is the $(m,n)$-entry of the matrix $U_\ell(\kappa)$. Using the Fourier transform, we can expand $W_{ij}(\kappa,\lambda)$ in a Fourier series as
\begin{align*}
    W_{ij}(\kappa, \lambda) &= W_{ij}(\text{Id}, \kappa^{-1} \lambda) 
    = \sum_{\ell \in \mathcal{I}} \text{dim} E_\ell \sum_{m = 1}^{\mathrm{dim}E_\ell} \sum_{n = 1}^{\mathrm{dim}E_\ell} \left[\widehat{W}_{ij}^{(\ell)} \right]_{mn} {U_{\ell}((\kappa^{-1}\lambda)^{-1})}_{mn}.  
\end{align*}
Next, form the block matrix $\widehat{W}^{(\ell)}$ by putting $\widehat{W}_{ij}^{(\ell)}$ in the $(i,j)$-th block. This is a matrix of size $N(\mathrm{dim}E_{\ell}) \times N(\text{dim}E_{\ell})$. 
\begin{lemma}\label{lem: hat(W)ell Hermitian}
The block matrix $\widehat{W}^{(\ell)}$ is Hermitian.
\end{lemma}
\begin{proof}
This follows from a direct computation, using the fact that $U_l(\kappa)$ is a unitary matrix.
\begin{align*}
  \left[\widehat{W}^{(\ell)}_{ij} \right]^\ast_{mn}  &= \int_K W_{ij}(Id, \kappa)U_\ell(\kappa)^\ast_{mn}d\kappa \\
  &= \int_K W_{ji}(Id, \kappa^{-1})U_\ell(\kappa^{-1})_{nm}d\kappa \\
   &= \int_K W_{ji}(Id, \kappa)U_\ell(\kappa)_{nm}d\kappa \\
   &= \left[ \widehat{W}^{(\ell)}_{ji} \right]_{nm}.
\end{align*}
\end{proof}
Form the matrix $S_\ell = D \otimes I_{\text{dim}E_{\ell}} - \widehat{W}^{(\ell)}$, where $D \otimes I_{\text{dim}E_{\ell}}$ denotes the Kronecker product of the matrices $D$ and $I_{\text{dim}E_{\ell}}$. As the following theorem proves, we can compute the eigenvalues and eigenvectors of the graph Laplacian $L$ from those of  $S_\ell$.
\begin{theorem}\label{thm: spectrum of L}
The graph Laplacian $L$ admits a sequence of nonnegative   eigenvalues $\{ \lambda_{\ell,1}, \dots, \lambda_{\ell, N(\text{dim}E_{\ell})}\}_{\ell \in \mathcal{I}}$ and a sequence of corresponding eigenfunctions $\{ \Phi_{m}^{(\ell,1)}, \dots, \Phi_{m}^{(\ell,N\mathrm{dim}E_{\ell})}\}_{\ell \in \mathcal{I}, 1\leq m \leq \mathrm{dim}E_\ell}$ which are orthogonal and complete over $\mathcal{L}^2(\mathcal{H})$ and are given by
\[ \Phi_{m}^{(\ell, s)}(i, \kappa) =  \sum_{n = 1}^{\mathrm{dim}E_\ell} U_\ell(\kappa)_{mn}v^{(\ell,s)}_{(i-1)\text{dim}E_\ell+ n},  \]
where $(\lambda_{\ell, s}, v^{(\ell,s)})$ is the $s$-th eigenvalue-eigenvector pair of the matrix $S_\ell = D \otimes I_{\text{dim}E_{\ell}} - \widehat{W}^{(\ell)}$, for $1\leq s \leq N \mathrm{dim}E_\ell$. 
\end{theorem}
\begin{proof}
We begin by noting that the matrix $S_\ell = D \otimes I_{\text{dim}E_{\ell}} - \widehat{W}^{(\ell)}$ is diagonalizable since, as proved in Lemma \ref{lem: hat(W)ell Hermitian}, $\widehat{W}^{(\ell)}$ is a Hermitian matrix. For  $1 \leq s \leq N \mathrm{dim}E_\ell,$ let $\left(\lambda_{\ell,s}, v^{(\ell,s)}\right)$ be the $s$-th eigenvalue-eigenvector pair of $S_\ell$, chosen so that the vectors $\{v^{(\ell,s)} \}_{1 \leq s \leq N \mathrm{dim}E_\ell}$ are orthonormal. By Lemma \ref{lem: form of eigenvalues of L} below,
$\Phi_{m}^{(\ell, s)}$ is an eigenfunction of $L$ with eigenvalue $\lambda_{\ell,s}$, where  $1 \leq s \leq N\mathrm{dim}E_\ell$ and $1 \leq m \leq \mathrm{dim}E_\ell$. 
Moreover, the eigenvalues $\lambda_{\ell,s}$ are nonnegative since $L$ is a positive semi-definite operator, as it is immediate from Claim \ref{Claim: quadratic form of L}.

Now let us prove the completeness of the eigenfunctions $\{ \Phi_{m}^{(\ell,1)}, \dots, \Phi_{m}^{(\ell,N\mathrm{dim}E_{\ell})}\}_{\ell, 1\leq m \leq \mathrm{dim}E_\ell}$. Note that
\[\mathcal{L}^2(\Gamma) = \mathcal{L}^2(K)\otimes \mathbb{R}^N = \bigoplus_{\ell \in \mathcal{I}} \left( E_\ell^{\oplus \mathrm{dim}E_\ell} \otimes \mathbb{R}^N\right),\]
where the second equality comes from the Peter-Weyl Theorem. For each $\ell$, let us see that the set $\{   \Phi_{m}^{(\ell,1)}, \dots, \Phi_{m}^{(\ell,N\text{dim}E_{\ell})}\}_{1\leq m \leq \mathrm{dim}E_\ell}$ is an orthogonal basis for the subspace $E_\ell^{\oplus \mathrm{dim}E_\ell} \otimes \mathbb{R}^N$ of $\mathcal{L}^2(\Gamma)$. We have $N(\mathrm{dim}E_\ell)^2 = \mathrm{dim}(E_\ell^{\oplus \mathrm{dim}E_\ell} \otimes \mathbb{R}^N)$ eigenfunctions, and the following computation proves they are orthogonal.
\begin{align*}
    \langle \Phi_m^{(\ell, s)} , \Phi_{m^\prime}^{(\ell, s^\prime)} \rangle &= \sum_{i=1}^N \int_K (\Phi_m^{(\ell, s)}(i, \kappa))^\ast \Phi_{m^\prime}^{(\ell, s^\prime)}(i,\kappa) d\kappa \\
    &= \sum_{i=1}^N \int_K \sum_{n = 1}^{\mathrm{dim}E_\ell} {U_\ell(\kappa)^\ast_{mn}} ({v^{(\ell,s)}_{(i-1)\text{dim}E_\ell+ n}})^\ast \sum_{n^\prime = 1}^{\mathrm{dim}E_\ell} U_\ell(\kappa)_{m^\prime n^\prime}v^{(\ell,s^\prime)}_{(i-1)\text{dim}E_\ell+ n^\prime} d\kappa \\
    &= \sum_{i=1}^N \sum_{n = 1}^{\mathrm{dim}E_\ell} \sum_{n^\prime = 1}^{\mathrm{dim}E_\ell}({v^{(\ell,s)}_{(i-1)\text{dim}E_\ell+ n}})^\ast v^{(\ell,s^\prime)}_{(i-1)\text{dim}E_\ell+ n^\prime}
    \int_K    U_\ell(\kappa^{-1})_{mn} U_\ell(\kappa)_{m^\prime n^\prime} d\kappa \\    
    &= \sum_{i=1}^N \sum_{n = 1}^{\mathrm{dim}E_\ell} \sum_{n^\prime = 1}^{\mathrm{dim}E_\ell}({v^{(\ell,s)}_{(i-1)\text{dim}E_\ell+ n}})^\ast v^{(\ell,s^\prime)}_{(i-1)\text{dim}E_\ell+ n^\prime} \frac{1}{\mathrm{dim}E_\ell} \delta_{m m^\prime} \delta_{n n^\prime}\\
    &= \frac{1}{\mathrm{dim}E_\ell}\delta_{m m^\prime} \sum_{i=1}^N \sum_{n = 1}^{\mathrm{dim}E_\ell} ({v^{(\ell,s)}_{(i-1)\text{dim}E_\ell+ n}} )^\ast v^{(\ell,s^\prime)}_{(i-1)\text{dim}E_\ell+ n^\prime}  \\
    &= \frac{1}{\mathrm{dim}E_\ell}\delta_{m m^\prime} \langle v^{(\ell,s)}, v^{(\ell,s^\prime)} \rangle_{\mathbb{R}^{N\text{dim}E_\ell}} \\
    &= \frac{1}{\mathrm{dim}E_\ell}\delta_{m m^\prime} \delta_{s s^\prime}.
\end{align*}
\end{proof}

\begin{lemma}\label{lem: form of eigenvalues of L}
If $v$ is an eigenvector of $S_\ell = D \otimes I_{\text{dim}E_{\ell}} - \widehat{W}^{(\ell)}$ with eigenvalue $\lambda$, then for all $ 1 \leq m \leq \mathrm{dim}E_\ell$ the function $\Phi_m \in \mathcal{L}^2(\mathcal{H})$ given by
\[ \Phi_m(i, \kappa) = \sum_{n = 1}^{\mathrm{dim}E_\ell} U_\ell(\kappa)_{mn}v_{(i-1)\mathrm{dim}E_\ell+ n}  \]
is an eigenfunction of $L$ with eigenvalue $\lambda$.
\end{lemma}
\begin{proof}
This follows from a direct computation.
\begin{align}\label{eq: D times Phi_m}
D\Phi_m(i,\kappa) = D_{ii}   \sum_{n = 1}^{\mathrm{dim}E_\ell} U_\ell(\kappa)_{mn}v_{(i-1)\text{dim}E_\ell+ n} =    \sum_{n = 1}^{\mathrm{dim}E_\ell} U_\ell(\kappa)_{mn}  D_{ii} v_{(i-1)\text{dim}E_\ell+ n}.
\end{align}
On the other hand,
\begin{align*}
W\Phi_m(i,\kappa) &= \sum_{j=1}^N \int_K W_{ij}(\kappa, \lambda)\Phi(j,\lambda)d\lambda \\
&= \sum_{j=1}^N \int_K W_{ij}(Id, \kappa^{-1}\lambda)  \sum_{n = 1}^{\mathrm{dim}E_\ell} U_\ell(\kappa)_{mn}v_{(j-1)\text{dim}E_\ell+ n} d\lambda \\
&= \sum_{j=1}^N \int_K \sum_{\ell \in \mathcal{I}^\prime} \text{dim} E_\ell \sum_{m^\prime = 1}^{\mathrm{dim}E_\ell} \sum_{n^\prime = 1}^{\mathrm{dim}E_\ell} \left[\widehat{W}_{ij}^{(\ell)} \right]_{m^\prime n^\prime} {U_{\ell^\prime}((\kappa^{-1}\lambda)^{-1})}_{m^\prime n^\prime}  \sum_{n = 1}^{\mathrm{dim}E_\ell} U_\ell(\lambda)_{mn}v_{(j-1)\text{dim}E_\ell+ n} d\lambda \\
&= \sum_{j=1}^N \sum_{\ell \in \mathcal{I}^\prime} \sum_{m^\prime = 1}^{\mathrm{dim}E_\ell} \sum_{n^\prime = 1}^{\mathrm{dim}E_\ell}  \sum_{n = 1}^{\mathrm{dim}E_\ell} {U_{\ell^\prime}(\kappa)}_{m^\prime n^\prime} \int_K \frac{\text{dim} E_\ell}{\mathrm{Vol}}(K)  \left[\widehat{W}_{ij}^{(\ell)} \right]_{m^\prime n^\prime}   {U_{\ell^\prime}(\lambda^{-1})}_{m^\prime n^\prime} U_\ell(\lambda)_{mn}v_{(j-1)\text{dim}E_\ell+ n} d\lambda \\
&= \sum_{j=1}^N \sum_{\ell \in \mathcal{I}^\prime} \sum_{m^\prime = 1}^{\mathrm{dim}E_\ell} \sum_{n^\prime = 1}^{\mathrm{dim}E_\ell}  \sum_{n = 1}^{\mathrm{dim}E_\ell} {U_{\ell^\prime}(\kappa)}_{m^\prime n^\prime}  \left[\widehat{W}_{ij}^{(\ell)} \right]_{m^\prime n^\prime}   v_{(j-1)\text{dim}E_\ell+ n} \delta_{\ell, \ell^\prime} \delta_{m, m^\prime} \delta_{n, n^\prime}  \\
&= \sum_{j=1}^N  \sum_{n = 1}^{\mathrm{dim}E_\ell} {U_{\ell}(\kappa)}_{m n}  \left[\widehat{W}_{ij}^{(\ell)} \right]_{m n}   v_{(j-1)\text{dim}E_\ell+ n}   \\
&=   \sum_{n = 1}^{\mathrm{dim}E_\ell} {U_{\ell}(\kappa)}_{m n}  \sum_{j=1}^N \left[\widehat{W}_{ij}^{(\ell)} \right]_{m n}   v_{(j-1)\text{dim}E_\ell+ n}.   \numberthis \label{eq: W times Phi_m}
\end{align*}
Combining Equations \eqref{eq: D times Phi_m} and \eqref{eq: W times Phi_m} gives
\begin{align*}
L\Phi_m(i,\kappa) &= (D-W)\Phi_m(i,\kappa) \\
&=  \sum_{n = 1}^{\mathrm{dim}E_\ell} U_\ell(\kappa)_{mn}\left[ D_{ii} v_{(i-1)\text{dim}E_\ell+ n} - \sum_{j=1}^N \left[\widehat{W}_{ij}^{(\ell)} \right]_{m n}   v_{(j-1)\text{dim}E_\ell+ n} \right] \\
&=  \sum_{n = 1}^{\mathrm{dim}E_\ell} U_\ell(\kappa)_{mn} \lambda v_{(i-1)\text{dim}E_\ell+ n} \\
&= \lambda \sum_{n = 1}^{\mathrm{dim}E_\ell} U_\ell(\kappa)_{mn}  v_{(i-1)\text{dim}E_\ell+ n}\\
&= \lambda \Phi_m(i, \kappa),
\end{align*}
as we wanted to see.
\end{proof}

\section{Convergence of the normalized $K$-invariant graph Laplacian to the Laplace-Beltrami operator} \label{sec: convergence of graph Laplacian}
The normalized graph Laplacian $L_N \colon \mathcal{H} \to \mathcal{H} $ is defined as
\[ L_N = D^{-1} L.  \]
The operator $L_N$ is generally not hermitian, but it is similar to the symmetric Laplacian $L_{\mathrm{sym}} = D^{-1/2} L D^{-1/2}$, a hermitian matrix, and hence $L_N$ is diagonalizable. More precisely, 
\[ L_N = I - D^{-1/2}(I - L_{\mathrm{sym}})D^{1/2}. \]
It follows that $L_N$ can be diagonalized with real eigenvalues; however, its eigenvectors will not generally be orthogonal. This is stated precisely in the following theorem.
\begin{theorem}\label{thm:eigen}
The normalized graph Laplacian $L_N$ admits a sequence of real eigenvalues $\{ \mu_{\ell,1}, \dots, \mu_{\ell, N(\text{dim}E_{\ell})}\}_{\ell \in \mathcal{I}}$ and a sequence of corresponding eigenfunctions $\{ \Psi_{m}^{(\ell,1)}, \dots, \Psi_{m}^{(\ell,N\mathrm{dim}E_{\ell})}\}_{\ell \in \mathcal{I}, 1\leq m \leq \mathrm{dim}E_\ell}$ which are complete over $\mathcal{L}^2(\mathcal{H})$ and are given by
\[ \Psi_{m}^{(\ell, s)}(i, \kappa) =  \sum_{n = 1}^{\mathrm{dim}E_\ell} U_\ell(\kappa)_{mn}w^{(\ell,s)}_{(i-1)\text{dim}E_\ell+ n},  \]
where $(\mu_{\ell, s}, w^{(\ell,s)})$ is the $s$-th eigenvalue-eigenvector pair of the matrix $S_{N,\ell} =  I_N \otimes I_{\text{dim}E_{\ell}} - \left(D^{-1}\otimes  I_{\text{dim}E_{\ell}}\right)\widehat{W}^{(\ell)}$, for $1\leq s \leq N \mathrm{dim}E_\ell$. 
\end{theorem}
\begin{proof}
This proof is identical to that of Theorem \ref{thm: spectrum of L}, except that we work with the matrix $\left(D^{-1}\otimes  I_{\text{dim}E_{\ell}}\right)\widehat{W}^{(\ell)}$ instead of $\widehat{W}^{(\ell)}$. Just as in the case of $L_N$, the matrix $\left(D^{-1}\otimes  I_{\text{dim}E_{\ell}}\right)\widehat{W}^{(\ell)}$  is similar to a Hermitian matrix. Thus it has a complete set of (not-necessarily orthogonal) eigenvectors $\{w^{(\ell,s)}\}_{1\leq s \leq N \mathrm{dim}E_\ell}$ with corresponding real eigenvalues $\{\mu_{\ell, s}\}_{1\leq s \leq N \mathrm{dim}E_\ell}$.
\end{proof}

Next, we show how the normalized graph Laplacian $L_N$ approximates the Laplace-Beltrami operator $\Delta_{\mathcal{M}}$ of the manifold $\mathcal{M}$ as the number of data points $N$ goes to infinity and the parameter $\varepsilon$ (used in Equation \eqref{eq: affinity of two points} to define the affinity operator) goes to zero. 
\begin{theorem} \label{thm:rate}
Let $\{x_1, \dots, x_N \} \subseteq \mathcal{M}$ be independent and identically distributed with uniform probability distribution $p(x) = 1/ \mathrm{\mathrm{Vol}}(\mathcal{M})$. If $f \colon \mathcal{M} \to \mathbb{R}$ is a smooth function, and if we define $g \in \mathcal{H}$ such that $g(i, \kappa) = f(\kappa \cdot x_i)$, then {with high probability} we have that
\[ \frac{4}{\epsilon} (L_Ng) (i, \kappa) = \Delta_{\mathcal{M}}f(\kappa \cdot x_i) + O(\varepsilon) + O\left(\frac{1}{N^{1/2}\varepsilon^{(d - \mathrm{dim}(K))/4 + 1/2}}\right).\]
\end{theorem}
\begin{proof}  We divide the proof of the theorem into two steps. The first step consists of taking the limit as $N \to \infty$, which yields the $O(\varepsilon)$ term, called the bias term. The second step is concerned with computing the remaining error term, also known as the variance term. \\

\noindent \underline{Step 1:} We have that
\begin{align*}
   \frac{4}{\varepsilon}\left(L_N g\right)(i, \kappa) 
    &= \frac{4}{\varepsilon}\left[ g(i, \kappa) - D_{ii}^{-1} \sum_{j=1}^N \int_K W_{ij}(\kappa, \lambda) g(j,\lambda) d\lambda \right]\\
    &= \frac{4}{\varepsilon}\left[ f(\kappa\cdot x_i) - \frac{ \frac{1}{N}\sum_{j=1}^N \int_K W_{ij}(\kappa, \lambda) f(\lambda\cdot x_j) d\lambda}{\frac{1}{N}\sum_{j=1}^N \int_K W_{ij}(\kappa, \lambda)d\lambda} \right]\\
    &= \frac{4}{\varepsilon}\left[ f(\kappa\cdot x_i) - \frac{ \frac{1}{N}\sum_{j=1}^N F_{i,\kappa}(x_j) }{\frac{1}{N}\sum_{j=1}^N G_{i,\kappa}(x_j) } \right],  \numberthis \label{eq: F/G}
\end{align*}
where we have defined
\[ F_{i,\kappa}(x) = \int_K \exp\left(-|| \kappa \cdot x_i - \lambda\cdot x ||^2/\varepsilon\right)f(\lambda\cdot x) d\lambda, \] 
and  similarly
\[ G_{i, \kappa}(x) = \int_K \exp\left(-|| \kappa \cdot x_i - \lambda\cdot x ||^2/\varepsilon\right) d\lambda. \]
Now let us take the limit as $N \to \infty$ of the numerator of the second term in Equation \eqref{eq: F/G}. By the law of large numbers, we get
\begin{align*}
    \lim_{N \to \infty} \frac{1}{N}\sum_{j=1}^N F_{i,\kappa}(x_j) &= \mathbb{E}\left[ F_{i,\kappa}\right] \\
    &= \frac{1}{\mathrm{Vol}(\mathcal{M})} \int_{\mathcal{M}} F_{i,\kappa}(x)dx \\
    &= \frac{1}{\mathrm{Vol}(\mathcal{M})} \int_{\mathcal{M}} \int_K \exp\left(-|| \kappa \cdot x_i - \lambda\cdot x ||^2/\varepsilon\right)f(\lambda\cdot x)  d\lambda dx \\
     &= \frac{1}{\mathrm{Vol}(\mathcal{M})}\int_{\mathcal{M}} \int_K \exp\left(-|| \kappa \cdot x_i - x^\prime ||^2/\varepsilon\right)f(x^\prime)  d\lambda dx^\prime \\
     &= \frac{1}{\mathrm{Vol}(\mathcal{M})}\int_{\mathcal{M}}  \exp\left(-|| \kappa \cdot x_i - x ||^2/\varepsilon\right)f(x)  dx. \numberthis \label{eq: limit F as N goes to infinity}
\end{align*}
Here we made the change of variables $x^\prime = \lambda \cdot x$. Note that $dx^\prime = dx$ since the Jacobian determinant of the change of coordinates induced by the map $\varphi_\lambda \colon \mathcal{M} \to \mathcal{M}$ with $\varphi_\lambda(x) = \lambda \cdot x$ is equal to $1$, as $K$ acts by isometries on $\mathbb{R}^{\mathcal{D}}$, and thus on $\mathcal{M}$. 

The calculation for $G_{i,\kappa}$ is the same as above, taking $f = 1$. Hence, we have
\begin{align*}
    \lim_{N \to \infty} \frac{1}{N}\sum_{j=1}^N G_{i,\kappa}(x_j) = \mathbb{E}\left[ G_{i,\kappa}\right] = \frac{1}{\mathrm{Vol}(\mathcal{M})}\int_{\mathcal{M}}  \exp\left(-|| \kappa \cdot x_i - x ||^2/\varepsilon\right)  dx. \numberthis \label{eq: limit G as N goes to infinity}
\end{align*}
Putting Equations \eqref{eq: limit F as N goes to infinity} and \eqref{eq: limit G as N goes to infinity} together gives
\begin{align*}
    \lim_{N \to \infty}  \frac{4}{\varepsilon}\left(L_N g\right)(i, \kappa) 
    &= \frac{4}{\varepsilon}\left[ f(\kappa\cdot x_i) - \frac{ \lim_{N \to \infty} \frac{1}{N}\sum_{j=1}^N F_{i,\kappa}(x_j) }{\lim_{N \to \infty} \frac{1}{N}\sum_{j=1}^N G_{i,\kappa}(x_j) } \right] \\
    &= \frac{4}{\varepsilon}\left[ f(\kappa\cdot x_i) - \frac{\frac{1}{\mathrm{Vol}(\mathcal{M})}\int_{\mathcal{M}}  \exp\left(-|| \kappa \cdot x_i - x ||^2/\varepsilon\right)f(x)  dx}{\frac{1}{\mathrm{Vol}(\mathcal{M})}\int_{\mathcal{M}}  \exp\left(-|| \kappa \cdot x_i - x ||^2/\varepsilon\right)  dx}\right]\\
    &= \Delta_{\mathcal{M}}f(\kappa \cdot x_i) + O(\varepsilon),
\end{align*}
where the last simplification is justified in \cite{Singer2006}. This concludes the proof of the first step.\\

\noindent \underline{Step 2:} Now we evaluate the variance term. To this end, we use Chernoff's inequality to estimate the probabilities
\begin{align}\label{eq: p_+N, alpha}
    p_+ (N, \alpha) &= \mathbb{P}\left[  \frac{ \frac{1}{N}\sum_{\substack{j=1 \\ j \neq i }}^N F_{i,\kappa}(x_j) }{\frac{1}{N}\sum_{\substack{j=1 \\ j \neq i }}^N G_{i,\kappa}(x_j) } - \frac{\mathbb{E}[F_{i,\kappa}]}{\mathbb{E}[G_{i,\kappa}]} > \alpha \right],
\end{align}
and 
\begin{align}\label{eq: p_-N, alpha}
    p_{-} (N, \alpha) &= \mathbb{P}\left[  \frac{ \frac{1}{N}\sum_{\substack{j=1 \\ j \neq i }}^N F_{i,\kappa}(x_j) }{\frac{1}{N}\sum_{\substack{j=1 \\ j \neq i }}^N G_{i,\kappa}(x_j) } - \frac{\mathbb{E}[F_{i,\kappa}]}{\mathbb{E}[G_{i,\kappa}]} < -\alpha \right].
\end{align}
Let us start with $p_+(N, \alpha)$. We can rewrite Equation \eqref{eq: p_+N, alpha} as
\begin{align}
    p_+ (N, \alpha) =  \mathbb{P}\left[ \sum_{j \neq i}^N Y_j > (N-1)\alpha(\mathbb{E}[G_{i,\kappa}])^2 \right],
\end{align}
where we have defined
\begin{align} \label{eq: Y_j}
     Y_j = \mathbb{E}[G_{i,\kappa}]F_{i,\kappa}(x_j) - \mathbb{E}[F_{i,\kappa}]G_{i,\kappa}(x_j) + \alpha \mathbb{E}[G_{i,\kappa}](\mathbb{E}[G_{i,\kappa}] - G_{i,\kappa}(x_j)).
\end{align}
The first and second moments of $Y_j$ are
$ \mathbb{E}[Y_j] = 0$
and
\begin{align}\label{eq: expectation of Yj}
    \mathbb{E}[Y_j^2] = (\mathbb{E}[G_{i,\kappa}])^2\mathbb{E}[F_{i,\kappa}^2(x_j)] - 2\mathbb{E}[F_{i,\kappa}(x_j)]\mathbb{E}[G_{i,\kappa}(x_j)]\mathbb{E}[F_{i,\kappa}(x_j)G_{i,\kappa}(x_j)] +  (\mathbb{E}[F_{i,\kappa}(x_j)])^2\mathbb{E}[G^2_{i,\kappa}(x_j)] + O(\alpha).
\end{align}
We are interested in finding an approximation to $\mathbb{E}[Y_j^2]$ in terms of $\varepsilon$ besides the current $O(\alpha)$ term. To do this, we need to approximate the different first and second moments involving $F_{i, \kappa}$ and $G_{i, \kappa}$. From \cite{Singer2006}, the first moments are given by
\begin{align}
    &\mathbb{E}[F_{i,\kappa}] = \frac{1}{\mathrm{Vol}(\mathcal{M})} \int_\mathcal{M}  \exp\left(-|| \kappa \cdot x_i - x ||^2/\varepsilon\right)f(x)  dx = \frac{1}{\mathrm{Vol}(\mathcal{M})} (\pi\varepsilon)^{d/2}[f(\kappa \cdot x_i) + O(\varepsilon) ], \label{eq: first moment F}  \\
   &\mathbb{E}[G_{i,\kappa}] = \frac{1}{\mathrm{Vol}(\mathcal{M})} \int_\mathcal{M}  \exp\left(-|| \kappa \cdot x_i - x ||^2/\varepsilon\right)  dx = \frac{1}{\mathrm{Vol}(\mathcal{M})} (\pi\varepsilon)^{d/2}[1 + O(\varepsilon) ]. \label{eq: first moment G} 
\end{align}
And from Lemma \ref{lemma: second moments of F and G} below, the second moments are
\begin{align}
    \mathbb{E}\left[F^2_{i,\kappa}\right] &= \frac{1}{\mathrm{Vol}(\mathcal{M})} \frac{(\pi \varepsilon)^{(d + \mathrm{dim}(K))/2}}{2^{(d - \mathrm{dim}(K))/2}}  \left[ \frac{ \nu(\kappa \cdot x_i)f^2(\kappa \cdot x_i)}{\mu^2(\kappa \cdot x_i)} + O(\varepsilon) \right], \\
    \mathbb{E}\left[G^2_{i,\kappa}\right] &= \frac{1}{\mathrm{Vol}(\mathcal{M})} \frac{(\pi \varepsilon)^{(d + \mathrm{dim}(K))/2}}{2^{(d - \mathrm{dim}(K))/2}} \left[ \frac{ \nu(\kappa \cdot x_i)}{\mu^2(\kappa \cdot x_i)} + O(\varepsilon) \right],\\    \mathbb{E}\left[F_{i,\kappa}G_{i,\kappa}\right] &= \frac{1}{\mathrm{Vol}(\mathcal{M})} \frac{(\pi \varepsilon)^{(d + \mathrm{dim}(K))/2}}{2^{(d - \mathrm{dim}(K))/2}}  \left[ \frac{ \nu(\kappa \cdot x_i)f(\kappa \cdot x_i)}{\mu^2(\kappa \cdot x_i)} + O(\varepsilon) \right]. \label{eq: second moment FG}
\end{align}
Substituting Equations \eqref{eq: first moment F} through \eqref{eq: second moment FG} into Equation \eqref{eq: expectation of Yj} gives
\begin{align*}
    \mathbb{E}[Y_j^2] = &\frac{1}{\mathrm{Vol}(\mathcal{M})^3} (\pi \varepsilon)^d \frac{(\pi \varepsilon)^{(d + \mathrm{dim}(K))/2}}{2^{(d - \mathrm{dim}(K))/2}}  \left[ \frac{\nu(\kappa \cdot x_i)f(\kappa \cdot x_i)}{\mu^2(\kappa \cdot x_i)} + O(\varepsilon) \right] \\
    &-2 \frac{1}{\mathrm{Vol}(\mathcal{M})^3} (\pi \varepsilon)^d \frac{(\pi \varepsilon)^{(d + \mathrm{dim}(K))/2}}{2^{(d - \mathrm{dim}(K))/2}}  \left[ \frac{\nu(\kappa \cdot x_i)f(\kappa \cdot x_i)}{\mu^2(\kappa \cdot x_i)} + O(\varepsilon) \right] \\
    &+  \frac{1}{\mathrm{Vol}(\mathcal{M})^3} (\pi \varepsilon)^d \frac{(\pi \varepsilon)^{(d + \mathrm{dim}(K))/2}}{2^{(d - \mathrm{dim}(K))/2}}  \left[ \frac{\nu(\kappa \cdot x_i)f(\kappa \cdot x_i)}{\mu^2(\kappa \cdot x_i)} + O(\varepsilon) \right] + O(\alpha) \\
    = &\frac{1}{\mathrm{Vol}(\mathcal{M})^3} (\pi \varepsilon)^d \frac{(\pi \varepsilon)^{(d + \mathrm{dim}(K))/2}}{2^{(d - \mathrm{dim}(K))/2}} O(\varepsilon) + O(\alpha) \\
    = & ~ O(\varepsilon^{(3d+\mathrm{dim}(K)+2)/2}) + O(\alpha). \numberthis \label{eq: approx of second moment of Yj}
\end{align*}
Next, we apply Chernoff's inequality to obtain an exponential bound on $p_+(N,\alpha)$ involving $\mathbb{E}[G_{i,\kappa}]$ and $\mathbb{E}[Y_j^2]$. This together with Equations \eqref{eq: first moment G} and \eqref{eq: approx of second moment of Yj} gives
\begin{align*}
    p_+ (N, \alpha) &\leq \exp\left( -\frac{\alpha^2(N-1)^2(\mathbb{E}[G_{i,\kappa}])^4}{2(N-1)\mathbb{E}[Y_j^2] + O(\alpha)} \right)\\
    &\leq  \exp\left( -\frac{\alpha^2 O(\varepsilon^{2d})}{O(\varepsilon^{(3d+\mathrm{dim}(K)+2)/2}/N) + O(\alpha)} \right) \\
    &\leq  \exp\left( -\frac{\alpha^2 }{O(\varepsilon^{(-d+\mathrm{dim}(K)+2)/2}/N) + O(\alpha)} \right). \numberthis \label{eq: Chernoff upper}
\end{align*}
To obtain a bound for $p_-(N, \alpha)$ we apply the same analysis, changing $\alpha$ by $-\alpha$ in the definition of $Y_j$ when we rewrite Equation \eqref{eq: p_-N, alpha} as we did for $p_+(N,\alpha)$. This and Chernoff's inequality for lower tails gives
\begin{align}
    p_- (N, \alpha)  
    &=  \mathbb{P}\left[ \sum_{j \neq i}^N Y_j <  -(N-1)\alpha(\mathbb{E}[G_{i,\kappa}])^2 \right] \leq  \exp\left( -\frac{\alpha^2 }{O(\varepsilon^{(-d+\mathrm{dim}(K)+2)/2}/N) + O(\alpha)} \right). \label{eq: Chernoff lower}
\end{align}
Inequalities \eqref{eq: Chernoff upper} and \eqref{eq: Chernoff lower} imply that by taking $\alpha = O(\varepsilon^{(-d+\mathrm{dim}(K)+2)/4}/\sqrt{N})$ we can make both $p_+(N, \alpha)$ and $p_-(N,\alpha)$ arbitrarily small with exponential decay. It follows that
\begin{align*}
    \left\vert \frac{ \frac{1}{N}\sum_{\substack{j=1 \\ j \neq i }}^N F_{i,\kappa}(x_j) }{\frac{1}{N}\sum_{\substack{j=1 \\ j \neq i }}^N G_{i,\kappa}(x_j) } - \frac{\mathbb{E}[F_{i,\kappa}]}{\mathbb{E}[G_{i,\kappa}]} \right\vert = \vert \alpha \vert 
    = O\left(\frac{1}{N^{1/2}\varepsilon^{(d - \mathrm{dim}(K))/4 - 1/2}}\right).
\end{align*}
To obtain the variance term we multiply the error above by $4/\varepsilon$. More precisely, we have
\begin{align*}
  \Bigg\vert \frac{4}{\varepsilon}\left[ f(\kappa\cdot x_i) - \frac{ \frac{1}{N}\sum_{\substack{j=1 \\ j \neq i }}^N F_{i,\kappa}(x_j) }{\frac{1}{N}\sum_{\substack{j=1 \\ j \neq i }}^N G_{i,\kappa}(x_j) } \right] &- \Big[\Delta_\mathcal{M}f(\kappa \cdot x_i) + O(\varepsilon) \Big]\Bigg\vert \\
  &= \left\vert \frac{4}{\varepsilon}\left[ f(\kappa\cdot x_i) - \frac{ \frac{1}{N}\sum_{\substack{j=1 \\ j \neq i }}^N F_{i,\kappa}(x_j) }{\frac{1}{N}\sum_{\substack{j=1 \\ j \neq i }}^N G_{i,\kappa}(x_j) } \right] - \frac{4}{\varepsilon}\left[f(\kappa\cdot x_i) - \frac{\mathbb{E}[F_{i,\kappa}]}{\mathbb{E}[G_{i,\kappa}]}\right]\right\vert \\
  &= \frac{4}{\varepsilon}\left\vert \frac{ \frac{1}{N}\sum_{\substack{j=1 \\ j \neq i }}^N F_{i,\kappa}(x_j) }{\frac{1}{N}\sum_{\substack{j=1 \\ j \neq i }}^N G_{i,\kappa}(x_j) } - \frac{\mathbb{E}[F_{i,\kappa}]}{\mathbb{E}[G_{i,\kappa}]} \right\vert  \\
  &= O\left(\frac{1}{N^{1/2}\varepsilon^{(d - \mathrm{dim}(K))/4 + 1/2}}\right).
\end{align*}
Finally, we prove that removing the diagonal terms (the case $j = i$) in the sums does not affect the convergence rate. Note that
\begin{align*}
    \frac{\sum_{j=1}^N F_{i,\kappa}(x_j)}{\sum_{j=1}^N G_{i,\kappa}(x_j)} - \frac{ \sum_{\substack{j=1 \\ j \neq i }}^N F_{i,\kappa}(x_j) }{\sum_{\substack{j=1 \\ j \neq i }}^N G_{i,\kappa}(x_j) } 
    &=  \frac{F_{i,\kappa}(x_i)}{\sum_{j=1}^N G_{i,\kappa}(x_j)} - \frac{G_{i,\kappa}(x_i)}{\sum_{j=1}^N G_{i,\kappa}(x_j)}\frac{ \sum_{\substack{j=1 \\ j \neq i }}^N F_{i,\kappa}(x_j) }{\sum_{\substack{j=1 \\ j \neq i }}^N G_{i,\kappa}(x_j) } \\
    &= O\left(\frac{\frac{1}{N}G_{i,\kappa}(x_i)}{\frac{1}{N}\sum_{j=1}^N G_{i,\kappa}(x_j)}\right) \\
    &= O\left(\frac{\varepsilon^{(\mathrm{dim}(K)/2)}/N}{\varepsilon^{d/2}}\right)\\
    &= O\left(\frac{1}{N\varepsilon^{(d-\mathrm{dim}(K))/2}}\right),
\end{align*}
where we used the fact that $f$ is a smooth function on $\mathcal{M}$ and thus uniformly bounded: $|f(x)| \leq c$ for some constant $c$, implying that $|F_{i,\kappa}(x_i) |\leq c |G_{i,\kappa}(x_i)|$. The error term we get is negligible compared to the variance term, so we ignore it. Therefore, we obtain
\begin{align*}
    \frac{4}{\varepsilon}(L_Ng)(i,\kappa) =  \frac{4}{\varepsilon}\left[ f(\kappa\cdot x_i) - \frac{ \frac{1}{N}\sum_{\substack{j=1 \\ j \neq i }}^N F_{i,\kappa}(x_j) }{\frac{1}{N}\sum_{\substack{j=1 \\ j \neq i }}^N G_{i,\kappa}(x_j) } \right] = \Delta_\mathcal{M}f(\kappa \cdot x_i) + O(\varepsilon)  +
    O\left(\frac{1}{N^{1/2}\varepsilon^{(d - \mathrm{dim}(K))/4 + 1/2}}\right),
\end{align*}
as we wanted to see.
\end{proof}

 We conclude this section by proving Lemma \ref{lemma: second moments of F and G}.
\begin{lemma} \label{lemma: second moments of F and G}
There exist  smooth function $\mu$ and $\nu$ on $\mathcal{M}$ such that 
\begin{align*}
    \mathbb{E}\left[F^2_{i,\kappa}\right] &= \frac{1}{\mathrm{Vol}(\mathcal{M})} \frac{(\pi \varepsilon)^{(d + \mathrm{dim}(K))/2}}{2^{(d - \mathrm{dim}(K))/2}}  \left[ \frac{\nu(\kappa\cdot x_i) f^2(\kappa \cdot x_i)}{\mu^2(\kappa \cdot x_i)} + O(\varepsilon) \right],\\
    \mathbb{E}\left[G^2_{i,\kappa}\right] &= \frac{1}{\mathrm{Vol}(\mathcal{M})} \frac{(\pi \varepsilon)^{(d + \mathrm{dim}(K))/2}}{2^{(d - \mathrm{dim}(K))/2}} \left[ \frac{ \nu(\kappa\cdot x_i)}{\mu^2(\kappa \cdot x_i)} + O(\varepsilon) \right],\\    \mathbb{E}\left[F_{i,\kappa}G_{i,\kappa}\right] &= \frac{1}{\mathrm{Vol}(\mathcal{M})} \frac{(\pi \varepsilon)^{(d + \mathrm{dim}(K))/2}}{2^{(d - \mathrm{dim}(K))/2}}  \left[ \frac{ \nu(\kappa\cdot x_i)f(\kappa \cdot x_i)}{\mu^2(\kappa \cdot x_i)} + O(\varepsilon) \right].
\end{align*}
\end{lemma}
\begin{proof}
We will prove the lemma for $\mathbb{E}\left[F^2_{i,\kappa}\right]$. The same proof holds for $\mathbb{E}\left[G^2_{i,\kappa}\right]$ and $\mathbb{E}\left[F_{i,\kappa}G_{i,\kappa}\right]$ by replacing $f$ by $1$ where necessary. We start by deriving an asymptotic expansion for 
\begin{equation} \label{eq: def of F}
    F_{i,\kappa}(x) =  \int_K \exp\left(-|| \kappa \cdot x_i - \lambda\cdot x ||^2/\varepsilon\right)f(\lambda\cdot x) d\lambda.
\end{equation}
To this end, we write
\begin{align*}
    || \kappa \cdot x_i - \lambda\cdot x ||^2 &= || (\kappa \cdot x_i -x) + (x - \lambda\cdot x) ||^2 \\
    &= || \kappa \cdot x_i -z ||^2 + 2\mathrm{Re}\langle \kappa \cdot x_i -x, x - \lambda\cdot x \rangle + || x - \lambda\cdot x||^2 \numberthis \label{eq: norm squared}
\end{align*}
and apply the Taylor expansion
\begin{align}\label{eq: Taylor expansion}
    \exp\left( 2\mathrm{Re}\langle \kappa \cdot x_i -x, x - \lambda\cdot x \rangle / \varepsilon \right) = 1 - O\left(\frac{2\mathrm{Re}\langle \kappa \cdot x_i -x, x - \lambda\cdot x \rangle }{\varepsilon}\right).
\end{align}
Plugging Equations \eqref{eq: norm squared} and \eqref{eq: Taylor expansion} into \eqref{eq: def of F} gives
\begin{align}\label{eq: expansion of F_i, kappa in two terms}
\begin{split}
    F_{i,\kappa}(x) = \exp\left(-|| \kappa \cdot x_i -x ||^2/\varepsilon\right) \Bigg[&\int_K \exp\left(-|| x - \lambda\cdot x||^2/ \varepsilon\right) f(\lambda\cdot x) d\lambda \\
    &- O\left(\frac{1}{\varepsilon} \int_K  2\mathrm{Re}\langle \kappa \cdot x_i -x, x - \lambda\cdot x \rangle  \exp\left(-|| x - \lambda\cdot x||^2/ \varepsilon\right) f(\lambda\cdot x) d\lambda \right)\Bigg].
\end{split}    
 \end{align}
Hence, we need to find asymptotic expansions for each of the two terms in Equation \eqref{eq: expansion of F_i, kappa in two terms}. Let us denote by $\mathcal{O}(x) = \{ \lambda\cdot x \mid \lambda \in K \}$ the orbit of $x \in \mathcal{M'}$ under the action of $K$, which is a smooth manifold of dimension equal to $ \mathrm{dim}(K)$. Using this, we can rewrite the integrals over $K$ in Equation \eqref{eq: expansion of F_i, kappa in two terms} as integrals over $\mathcal{O}(x)$. For the first integral, we obtain
\begin{align}\label{eq: integral exp f}
\begin{split}
    \int_K \exp\left(-|| x - \lambda\cdot x||^2/ \varepsilon\right) f(\lambda\cdot x) d\lambda &= \frac{1}{\mu(x)} \int_{\mathcal{O}(x)} \exp\left(-|| x - y||^2/ \varepsilon\right) f(y) dy, 
\end{split}
\end{align}
where $\mu(x) = |D\varphi_x|$ is the Jacobian determinant of the map $\varphi_x \colon K \to \mathcal{O}(x)$ defined by $\varphi_x(\lambda) = \lambda\cdot x$, which is a diffeomorphism \cite[Proposition 21.7]{lee2012smooth}. Next, applying Proposition 9 in \cite{SteerablePaper} allows us to approximate the left-hand-side integral in Equation \eqref{eq: integral exp f} as
\begin{equation}\label{eq: final integral exp f}
     \int_K \exp\left(-|| x - \lambda\cdot x||^2/ \varepsilon\right) f(\lambda\cdot x) d\lambda = \frac{(\pi \varepsilon)^ {\mathrm{dim}(K)/2}}{\mu(x)}[f(x) + O(\varepsilon)].
\end{equation}
We do the same for the second integral in \eqref{eq: expansion of F_i, kappa in two terms}, which gives
\begin{align}\label{eq: integral Re exp f}
\begin{split}
    O&\left(\frac{1}{\varepsilon} \int_K  2\mathrm{Re}\langle \kappa \cdot x_i -x, x - \lambda\cdot x \rangle  \exp\left(-|| x - \lambda\cdot x||^2/ \varepsilon\right) f(\lambda\cdot x) d\lambda \right)\\
    &= O\left(\frac{1}{\varepsilon \mu(x)} \int_{\mathcal{O}(x)}  2\mathrm{Re}\langle \kappa \cdot x_i -x, x -  y\rangle  \exp\left(-|| x - y||^2/ \varepsilon\right) f(y) dy \right)\\
    &=O\left(\frac{(\pi \varepsilon)^ {\mathrm{dim}(K)/2}}{ \mu(x)} \left[ \frac{1}{2} \Delta_{\mathcal{O}(x)}(\mathrm{Re}\langle \kappa \cdot x_i -x, x -  y\rangle f(y))\mid_{y=x} + \varepsilon \right] \right).
\end{split}    
\end{align}
Note that
\begin{align} 
    \begin{split}
        \Delta_{\mathcal{O}(x)}(\mathrm{Re}\langle \kappa \cdot x_i -x, x -  y\rangle f(y))\mid_{y=x} = ~&\Delta_{\mathcal{O}(x)}f(y)\vert_{y=x} \mathrm{Re}\langle \kappa \cdot x_i -x, x -  x\rangle \\
        &-2 \langle \nabla_{\mathcal{O}(x)}\mathrm{Re}\langle \kappa \cdot x_i -x, x -  y\rangle\vert_{y=x}, \nabla_{\mathcal{O}(x)}f(z) \rangle \\
        &+ f(x)\Delta_{\mathcal{O}(x)}\mathrm{Re}\langle \kappa \cdot x_i -x, x -  y\rangle\vert_{y=x}  \\ 
        =  & ~ 2 \Big\langle \mathrm{Re} \langle \kappa \cdot x_i -x, \nabla_{\mathcal{O}(x)} y \vert_{y=x} \rangle, \nabla_{\mathcal{O}(x)}f(z) \Big\rangle \\
        &- f(x)  \mathrm{Re}\langle \kappa \cdot x_i -x, \Delta_{\mathcal{O}(x)} y\vert_{y=x} \rangle.
    \end{split}
\end{align} 
To simplify the notation, let us define
\[ q(x) =  \Big\langle \mathrm{Re} \langle \kappa \cdot x_i -x, \nabla_{\mathcal{O}(x)} y \vert_{y=x} \rangle, \nabla_{\mathcal{O}(x)}f(z) \Big\rangle - \frac{f(x)}{2} \mathrm{Re}\langle \kappa \cdot x_i -x, \Delta_{\mathcal{O}(x)} y\vert_{y=x} \rangle. \]
This function satisfies $q(\kappa \cdot x_i) = 0$. Therefore, Equation \eqref{eq: integral Re exp f} can be written as
\begin{align}\label{eq: final integral Re exp f}
    O&\left(\frac{1}{\varepsilon} \int_K  2\mathrm{Re}\langle \kappa \cdot x_i -x, x - \lambda\cdot x \rangle  \exp\left(-|| x - \lambda\cdot x||^2/ \varepsilon\right) f(\lambda\cdot x) d\lambda \right) = O\left( \frac{(\pi \varepsilon)^ {\mathrm{dim}(K)/2}}{ \mu(x)}[q(x) + \varepsilon] \right).
\end{align}
Plugging in Equations \eqref{eq: final integral Re exp f} and \eqref{eq: final integral exp f} into \eqref{eq: expansion of F_i, kappa in two terms} yields the asymptotic expansion
\begin{align*}
    F_{i,\kappa}(x) = \exp\left(-|| \kappa \cdot x_i -x ||^2/\varepsilon\right) \frac{(\pi \varepsilon)^ {\mathrm{dim}(K)/2}}{\mu(x)}[f(x) + O(q(x)) + O(\varepsilon)]. 
\end{align*}
From here we directly obtain an approximation for $(F_{i,\kappa}(x))^2$, namely
\begin{align}\label{eq: approx F squared}
    (F_{i,\kappa}(x))^2 = \exp\left(-2|| \kappa \cdot x_i -x ||^2/\varepsilon\right) \frac{(\pi \varepsilon)^ {\mathrm{dim}(K)}}{\mu^2(x)}[f^2(x) + O(q(x)) + O(\varepsilon)].
\end{align}
The final step is to use Formula \eqref{eq: approx F squared} to estimate the expected value of $F_{i,\kappa}^2$. To integrate this function over $\mathcal{M}$, consider the quotient map $\pi \colon \mathcal{M}^\prime \to \mathcal{M}^\prime/K$, which is a smooth submersion by the quotient manifold theorem \cite{lee2012smooth}; moreover, $\mathcal{M}^\prime/K$ is a smooth manifold of dimension equal to $d - \mathrm{dim}(K)$. It inherits a natural Riemannian metric such that if $g$ is an integrable function on $\mathcal{M}^\prime/K$, then it holds that
\begin{align*}
    \int_{\mathcal{M}^\prime} \pi^\ast g ~\mu_{\mathcal{M}^\prime} = \int_{\mathcal{M}^\prime/K} g ~ \pi_\ast\mu_{\mathcal{M}^\prime},
\end{align*}
where $\mu_{\mathcal{M}^\prime}$ is the induced measure on $\mathcal{M}^\prime$ and $\pi^\ast$, $\pi_\ast$ denote pullback and pushforward respectively. 
To apply this identity, note that the function $F_{i,\kappa}$ is invariant under the action of $K$. Indeed, for any $x \in \mathcal{M}$ and $\eta \in K$ we have that
\begin{align*}
    F_{i, \kappa}(\eta\cdot x) &= \int_K \exp\left(-|| \kappa \cdot x_i - \lambda\cdot \eta\cdot x ||^2/\varepsilon\right) f(\lambda\cdot \eta\cdot x) d\lambda\\
    &= \int_K \exp\left(-|| \kappa \cdot x_i - \lambda^\prime \cdot x ||^2/\varepsilon\right) f(\lambda^\prime\cdot x) d(\eta^{-1}\lambda^\prime)\\
    &= \int_K \exp\left(-|| \kappa \cdot x_i - \lambda^\prime \cdot x ||^2/\varepsilon\right) f(\lambda^\prime\cdot x) d(\lambda^\prime)\\
    &= F_{i, \kappa}(x).
\end{align*}
It follows that $F_{i, \kappa} = \pi^\ast g$ for some function $g$ on $\mathcal{M}^\prime/K$.
Let $\nu$ be the function on $\mathcal{M}^\prime$ given by
$\nu(x) \colonequals \sqrt{|\mathrm{det}(g_\mathcal{M}^\prime)(x)|}$, where $g_\mathcal{M}^\prime$ is the metric on $\mathcal{M}^\prime$, comes from the volume form of the manifold $\mathcal{M}^\prime$. This and Formula \eqref{eq: approx F squared} yield
\begin{align*}
    \mathbb{E}\left[ (F_{i,\kappa}(x))^2\right] &= \frac{1}{\mathrm{Vol}(\mathcal{M})} \int_\mathcal{M} (F_{i,\kappa}(x))^2 dx \\
    &= \frac{1}{\mathrm{Vol}(\mathcal{M})} \int_{\mathcal{M}^\prime/K} (F_{i,\kappa}( x))^2 \nu(x) dx  \\
    &=\frac{1}{\mathrm{Vol}(\mathcal{M})} (\pi \varepsilon)^ {\mathrm{dim}(K)} \int_{\mathcal{M}^\prime/K} \exp\left(-2|| \kappa \cdot x_i -x ||^2/\varepsilon\right) \frac{\nu(x)}{\mu^2(x)}[f^2(x) + O(q(x) + \varepsilon)] dx. \numberthis \label{eq: internal over N}
\end{align*}
Again, we evaluate all terms in the  integral \eqref{eq: internal over N} using Proposition 9 in \cite{SteerablePaper}. We obtain
\begin{align}\label{eq: first term int over N}
    \int_{\mathcal{M}/K} \exp\left(-2|| \kappa \cdot x_i -x ||^2/\varepsilon\right) \frac{\nu(x) f^2(x)}{\mu^2(x)} dx = {(\pi \varepsilon/2)^{(d - \mathrm{dim}(K))/2}} \left[ \frac{ \nu(\kappa \cdot x_i)f^2(\kappa \cdot x_i)}{\mu^2(\kappa \cdot x_i)} + O(\varepsilon) \right],
\end{align}
and
\begin{align}\label{eq: second term int over N}
    O\left(\int_{\mathcal{M}/K} \exp\left(-2|| \kappa \cdot x_i -x ||^2/\varepsilon\right) \frac{\nu(x)q(x) }{\mu^2(x)} dx\right) &= O\left({(\pi \varepsilon)^{(d - \mathrm{dim}(K))/2}}\left[ \frac{ \nu(\kappa \cdot x_i) q(\kappa \cdot x_i)}{\mu^2(\kappa \cdot x_i)} + \varepsilon \right] \right)\\
    &= {(\pi \varepsilon/2)^{(d - \mathrm{dim}(K))/2}} O(\varepsilon)  
\end{align}
since $q(\kappa \cdot x_i)= 0$. By plugging Equations \eqref{eq: first term int over N} and \eqref{eq: second term int over N} into \eqref{eq: internal over N}, it follows that
\begin{align}
\begin{split}
   \mathbb{E}\left[ (F_{i,\kappa}(x))^2\right] 
    &= \frac{1}{\mathrm{Vol}(\mathcal{M})} (\pi \varepsilon)^ {\mathrm{dim}(K)} {(\pi \varepsilon/2)^{(d - \mathrm{dim}(K))/2}} \left[ \frac{ \nu(\kappa \cdot x_i)f^2(\kappa \cdot x_i)}{\mu^2(\kappa \cdot x_i)} + O(\varepsilon) \right] \\
    &= \frac{1}{\mathrm{Vol}(\mathcal{M})} \frac{(\pi \varepsilon)^{(d + \mathrm{dim}(K))/2}}{2^{(d - \mathrm{dim}(K))/2}} \left[ \frac{ \nu(\kappa \cdot x_i)f^2(\kappa \cdot x_i)}{\mu^2(\kappa \cdot x_i)} + O(\varepsilon) \right], 
\end{split}
\end{align}
concluding the proof.
\end{proof}

\section{Concluding remarks}
In this paper, we developed the framework of $K$-invariant graph Laplacian operators $L$ for manifold learning under group actions.  
Our work extends a method of Landa and Shkolnisky from the 2D rotation group to the case of an arbitrary compact Lie group $K$.  
The $K$-invariant graph Laplacian provides a principled approach to manifold learning when the data manifold is invariant under the action of the group $K$.
The method achieves this by carrying out infinite data augmentation, through analytic integration over the orbits of the data points.  
In Theorem~\ref{thm:eigen} we provided an explicit formula for computing the eigenvalues and eigenfunctions of $L$ using the representation theory of $K$. Moreover, in Theorem~\ref{thm:rate}, we proved that the normalized $K$-invariant graph Laplacian operator $L_N$ converges to the Laplace-Beltrami operator $\Delta_\mathcal{M}$ of the data manifold $\mathcal{M}$, at a faster rate than the symmetry-unaware graph Laplacian.  In particular, the variance term involves the dimension of $K$. 
In practice, this result will yield an improvement in sample complexity. 

There are several directions worthy of future work. First and foremost, we wish to use the theoretical apparatus of this paper in a real-life application.  To this end, mapping out low-dimensional models for molecular conformation spaces in cryo-electron tomography \cite{wan2016cryo} is a target application.  In this case, data points are naturally subject to an $\operatorname{SO}(3)$-action. 
Secondly, a computational challenge in implementing the approach will come from implementing the Fourier transform of each $W_{ij}(\kappa, \lambda)$ from the block matrices $\widehat{W}^\ell$ since, depending on the group, the $\mathrm{dim}E_\ell$ could grow faster than desired \cite{chirikjian2016harmonic}.
Thirdly, it is natural to ask about spectral convergence \cite{calder2022improved} rather than the pointwise convergence established here. 
That is, we would like to show the convergence of the eigenfunctions of the $K$-invariant graph Laplacian operator to the eigenfunctions of the Laplace-Beltrami operator on $\mathcal{M}$.  

\bibliographystyle{abbrv}
\bibliography{references}
\end{document}